\documentclass[11pt]{article}
\usepackage[utf8]{inputenc}
\usepackage{fullpage}
\usepackage{amsmath,amssymb,amsthm,mathtools,xcolor,bm,hyperref,graphicx,float}
\usepackage{import}
\usepackage[round]{natbib}
\usepackage{authblk}
\graphicspath{{figs/}}

\newcommand{\bx}{\mathbf{x}}
\newcommand{\by}{\mathbf{y}}
\newcommand{\bz}{\mathbf{z}}
\newcommand{\bbeta}{\boldsymbol{\beta}}
\newcommand{\bdelta}{\boldsymbol{\delta}}

\newcommand{\blambda}{\boldsymbol{\lambda}}
\newcommand{\bpi}{\boldsymbol{\pi}}

\newcommand{\R}{\mathbb{R}}
\renewcommand{\P}{\mathbb{P}}
\newcommand{\E}{\mathbb{E}}
\newcommand{\indicator}{\mathbf{1}}
\newcommand{\Var}{\textnormal{Var}}

\DeclarePairedDelimiter{\braces}{\lbrace}{\rbrace}
\DeclarePairedDelimiter{\paren}{(}{)}

\DeclarePairedDelimiter{\abs}{|}{|}
\DeclarePairedDelimiter{\norm}{\|}{\|}

\newcommand{\data}[1][n]{\mathcal{D}_{#1}}

\newcommand{\partition}{\mathfrak{p}}

\newcommand{\cell}{\mathcal{C}}

\renewcommand{\dim}{d}

\newcommand{\coordindices}{[\dim]}

\newcommand{\hypercube}[1][d]{\braces{0,1}^{#1}}
\newcommand{\uniform}{\mu}
\newcommand{\measure}{\nu}
\newcommand{\risk}{\mathcal{R}}
\newcommand{\oraclerisk}{\mathcal{R}^*}

\newcommand\blfootnote[1]{%
  \begingroup
  \renewcommand\thefootnote{}\footnote{#1}%
  \addtocounter{footnote}{-1}%
  \endgroup
}

\newtheorem{theorem}{Theorem}[section]
\newtheorem{lemma}[theorem]{Lemma}
\newtheorem{proposition}[theorem]{Proposition}

\title{A cautionary tale on fitting decision trees to data from additive models: generalization lower bounds}
\date{\today}

\author[1]{Yan Shuo Tan\footnote{To whom correspondence should be addressed. E-mail: yanshuo@berkeley.edu}}
\author[2]{Abhineet Agarwal}
\author[1, 3, 4, 5]{Bin Yu}
\affil[1]{Department of Statistics, UC Berkeley}
\affil[2]{Department of Physics, UC Berkeley}
\affil[3]{Department of Electrical Engineering and Computer Sciences, UC Berkeley}
\affil[4]{Center for Computational Biology, UC Berkeley}
\affil[5]{Chan-Zuckerberg Biohub Intercampus Award Investigator}

\begin{document}

\maketitle

\begin{abstract}
    Decision trees are important both as interpretable models amenable to high-stakes decision-making, and as building blocks of ensemble methods such as random forests and gradient boosting. Their statistical properties, however, are not well understood. The most cited prior works have focused on deriving pointwise consistency guarantees for CART in a classical nonparametric regression setting. We take a different approach, and advocate studying the generalization performance of decision trees with respect to different generative regression models. This allows us to elicit their \emph{inductive bias}, that is, the assumptions the algorithms make (or do not make) to generalize to new data, thereby guiding practitioners on when and how to apply these methods. In this paper, we focus on sparse additive generative models, which have both low statistical complexity and some nonparametric flexibility. We prove a sharp squared error generalization lower bound for a large class of decision tree algorithms fitted to sparse additive models with $C^1$ component functions. This bound is surprisingly much worse than the minimax rate for estimating such sparse additive models. The inefficiency is due not to greediness, but to the loss in power for detecting global structure when we average responses solely over each leaf, an observation that suggests opportunities to improve tree-based algorithms, for example, by hierarchical shrinkage. To prove these bounds, we develop new technical machinery, establishing a novel connection between decision tree estimation and rate-distortion theory, a sub-field of information theory.
\end{abstract}

\section{Introduction}


Using decision trees for supervised learning has a long and storied history. First introduced by \citet{morgan1963problems}, the idea is simple: recursively split your covariate space along coordinate directions, and fit a piecewise constant model on the resulting partition. The \emph{adaptivity} of splits to structure in the data improves conciseness and statistical efficiency of tree models. Meanwhile, the \emph{greedy splitting} principle followed by most algorithms, including \cite{breiman1984classification}'s Classification and Regression Trees (CART), ensures computational tractability. More recently, there has also been growing interest in fitting optimal decision trees using mathematical programming or dynamic programming techniques \citep{lin2020generalized,aghaei2021strong}.\blfootnote{Code and documentation for easily reproducing the results are provided at \url{https://github.com/aagarwal1996/additive_trees}.}


Decision tree models are important for two main reasons. First, shallow decision trees are interpretable models \citep{rudin2021interpretable}: They can be implemented by hand, and they are easily described and visualized. While the precise definition and utility of interpretability has been a subject of much debate \citep{murdoch2019definitions,doshi2017towards,rudin2019stop}, all agree that it is an important supplement to prediction accuracy in high-stakes decision-making such as medical risk assessment and criminal justice. For this reason, decision trees have been widely applied in both areas \citep{steadman2000classification, kuppermann2009identification, letham2015interpretable, angelino2017learning}. Second, CART trees are used as the basic building blocks of ensemble machine learning algorithms such as random forests (RF) and gradient boosting \citep{breiman2001random, friedman2001greedy}. These algorithms are recognized as having state-of-the-art performance over a wide class of prediction problems \citep{caruana2006empirical,caruana2008empirical,fernandez2014we,olson2018data,hooker2021bridging}, and receive widespread use, given their implementation in popular machine learning  packages such as \texttt{ranger} \citep{wright2015ranger}, \texttt{scikit-learn} \citep{pedregosa2011scikit} and \texttt{xgboost} \citep{chen2016xgboost}. Random forests in particular have also shown promise in scientific applications, for example in discovering interactions in genomics \citep{boulesteix2012overview,basu2018iterative}.

Because of the centrality of decision trees in the machine learning edifice, it is all the more surprising that there has been relatively little theory on their statistical properties. In the regression setting, some of the most cited prior works have focused on deriving pointwise \emph{consistency} guarantees for CART when assuming that the conditional mean function is Lipschitz continuous \citep{biau2012analysis,wager2018}. Unfortunately, each is forced to modify the splitting criterion in the algorithm to ensure that the mesh of the learnt partition shrinks to zero.
\cite{scornet2015} proved the first consistency result for the unmodified CART algorithm by replacing the fully nonparametric regression model with an additive regression model \citep{friedman2001elements}. This generative assumption simplifies calculations by avoiding some of the complex dependencies between splits that may accumulate during recursive splitting. Moreover, it prevents the existence of locally optimal trees that are not globally optimal, which would otherwise trip up greedy methods such as CART. \cite{klusowski2020,klusowski2021universal} has extended this analysis to sparse additive models, showing that when the true conditional mean function depends only on a fixed subset of $s$ covariates, CART is still consistent even when the total number of covariates is allowed to grow exponentially in the sample size. This adaptivity to sparsity somewhat alleviates the curse of dimensionality, and partially explains why CART and RF are often preferred in practice to $k$-nearest neighbors.

As natural generalizations of linear models, additive models simultaneously have low statistical complexity and yet sufficient nonparametric flexibility required to describe some real world datasets well. Moreover, if the component functions are not too complex, additive models have aspects of interpretability \citep{rudin2021interpretable}. Unsurprisingly, they have accumulated a rich statistical literature \citep{hastie1986generalized,sadhanala2019additive}. While the previously discussed works have proved consistency for CART on additive regression models, it is also important to compute \emph{rate upper and lower bounds} for the generalization error of CART and other decision tree algorithms. This would allow us to compare their performance with that of specially tailored algorithms such as backfitting \citep{hastie1995penalized}, and hence understand whether the inductive biases of decision trees are able to fully exploit the structure present in additive models.


\subsection{Main contributions}

In this paper, we provide generalization lower bounds for a large class of decision trees, which we call ALA, when fitted to data generated from sparse additive models. We define an ALA tree as one that learns an \textbf{a}xis-aligned partition of the covariate space, and makes predictions by averaging the responses over each leaf. We call this second aspect \emph{\textbf{l}eaf-only \textbf{a}veraging}. In addition, we will assume for analytical reasons that our trees are \emph{honest}, which means that one sample is used to learn the partition, and a separate sample is used to estimate the averages over each leaf \citep{athey2016recursive}. CART is an example of an ALA tree, and so are most (but not all) decision tree algorithms used in practice. The reason we consider this level of generality is to remove the effect of greediness that has dominated the analysis of CART thus far, and to argue that leaf-only averaging subtly introduces its own inductive bias.

We show that when the true conditional mean function is a sparse additive model with $s$ $C^1$ univariate component functions, no honest ALA tree, even one that has oracle access to the true conditional mean function, can perform better than $\Omega\paren*{n^{-\frac{2}{s+2}}}$ in expected $\ell_2$ risk. This is the $\ell_2$ minimax rate for nonparametric estimation of $C^1$ functions in $s$ dimensions \citep{stone1982optimal}. In contrast, if each univariate component function in the model is assumed to be $C^1$, the minimax rate for sparse additive models scales as $\max\braces*{\frac{s\log(d/s)}{n},\frac{s}{n^{2/3}}}$ \citep{raskutti2012minimax}. As such, while it is possible to achieve a prescribed error tolerance with $\tilde{O}\paren*{s^{3/2}}$ samples via convex programming, ALA trees have a sample complexity that is at least exponential in $s$, which means that they needlessly suffer from the curse of dimensionality. The intuitive explanation for this inefficiency is that by ignoring information from other leaves when making a prediction, leaf-only averaging creates an inductive bias \emph{against} global structure.

As far as we know, this paper is the first to establish algorithm-specific lower bounds for CART or any other decision tree algorithm. More broadly, algorithm-specific lower bounds can be challenging in the machine learning literature because they require specialized techniques instead of relying on a general recipe (as is the case with minimax lower bounds). Additionally, we show that the rate lower bound is achievable using an oracle partition. We also obtain a lower bound for additive models over Boolean features, which surprisingly, has a very different form. We note that \cite{tang2018random} proved sufficient conditions under which honest random forest estimators are inconsistent for special regression functions using \cite{stone1977consistent}'s adversarial construction. This construction does not produce additive functions. Theirs is the only other work we know of that provides negative results for tree-based estimators. On the other hand, they do not compute lower bounds, and their conditions either involve unrealistic choices of hyperparameters, such as requiring each tree to only use a constant number of samples, or pertain to properties of trees after they are grown, such as upper bounds on the rate of shrinkage of leaf diameters. It is not clear if or when these conditions hold in practice.

Our results are obtained using novel technical machinery, which are based on two simple insights: First, we show that the variance term of the expected $\ell_2$ risk scales \emph{linearly} with the number of leaf nodes. Second, a tree model can be thought of as a \emph{lossy code}, in which the number of leaves is the \emph{size} of the code, while the bias term of the risk is simply its \emph{distortion}. This link to rate-distortion theory, a sub-field of information theory \citep{cover2012elements}, allows us to compute the optimal trade-off between bias and variance to obtain lower bounds. As a happy by-product of our analysis, the first insight yields a better understanding of cost-complexity pruning and minimum impurity decrease procedures that are commonly used with CART to prevent overfitting.

\subsection{Other related work}

Here, we discuss some other theoretical work on CART that is less directly related to this paper. \cite{syrgkanis2020estimation} proved generalization upper bounds for CART in a different setting. They considered Boolean features, and imposed some type of submodularity assumption on the conditional mean function. While this subsumes additive models, the authors did not give concrete examples of other models satisfying this assumption. \cite{scornet2020trees} returned to the additive model setting, and was able to compute explicit asymptotic formulas for the popular mean impurity decrease (MDI) feature importance score. \cite{behr2021provable} formulated a biologically-inspired discontinuous nonlinear regression model, and showed that CART trees can be used to do inference for the model. Finally, we refer the reader to several excellent survey papers for a fuller description of the literature \citep{loh2014fifty,biau2016random,hooker2021bridging}.

\section{Preliminaries} \label{sec:preliminaries}

We work with the standard regression framework in supervised learning, and assume a generative model
\begin{equation} \label{eq:regression_model}
    y = f(\bx) + \epsilon
\end{equation}
where the feature vector $\bx$ is drawn from a distribution $\measure$ on a subset $\mathcal{X} \subset \R^\dim$, while the responses $y$ are real-valued, and $\epsilon$ is a noise variable that is mean zero when conditioned on $\bx$.\footnote{Here, we note that we will use bold small-case roman letters such as $\bx$ and $\bz$ denote vectors, while regular letters will denote scalars, indices, and functions, depending on the context.} We assume that the noise is homoskedastic, and denote $\sigma^2 \coloneqq \E\braces*{\epsilon^2~|~\bx}$. In this paper, $\mathcal{X}$ will either be the unit-length cube $[0,1]^\dim$ or the hypercube $\hypercube$. An \emph{additive model} is one in which we can decompose the conditional mean function as the sum of univariate functions along each coordinate direction:
\begin{equation} \label{eq:additive_model}
    f(\bx) = \sum_{j=1}^\dim \phi_j(x_j).
\end{equation}

We are given a training set $\data = \braces{(\bx^{(1)},y^{(1)}),\ldots (\bx^{(n)},y^{(n)})}$ comprising independent samples that are drawn according to the model \eqref{eq:regression_model}.\footnote{Sample indices will be denoted using superscripts, while subscripts will be reserved for coordinate indices.} Throughout this paper and unless mentioned otherwise, we will use $\P$, $\E$ and $\Var$ to denote quantities related to the population distribution generating \eqref{eq:regression_model}.

A \emph{cell} $\cell \subset \mathcal{X}$ is a rectangular subset. If $\mathcal{X} = [0,1]^\dim$, this means that it can be written as a product of intervals:
$$
\cell = [a_1,b_1] \times [a_2,b_2] \times \cdots \times [a_d,b_d].
$$
If $\mathcal{X} = \hypercube$, this means that it is a subcube of the form $\cell(S,\bz) = \braces{\bx \in \hypercube ~\colon~ x_j = z_j~\text{for}~j \in S}$ where $S \subset \coordindices$ is a subset of coordinate indices. Given a cell $\cell$ and a training set $\data$, let $N(\cell)\coloneqq \abs*{\braces*{i ~\colon~\bx^{(i)} \in \cell}}$ denote the number of samples in the cell.

A \emph{partition} $\partition = \braces{\cell_1,\ldots,\cell_j}$ is a collection of cells with disjoint interiors,  whose union is the entire space $\mathcal{X}$. Given the training set $\data$, every partition yields an estimator $\hat f(-; \partition, \data)$ for $f$ via \emph{leaf-only averaging}: For every input $\bx$, the estimator outputs the mean response over the cell containing $\bx$. In other words, we define
$$
\hat f(\bx; \partition, \data) \coloneqq \sum_{\cell \in \partition} \paren*{ \frac{1}{N(\cell)}\sum_{\bx^{(i)} \in \cell} y^{(i)}}~ \indicator\braces{\bx \in \cell}.
$$
We will use the convention that if $N(\cell)=0$, then we set $\frac{1}{N(\cell)}\sum_{\bx^{(i)} \in \cell} y^{(i)} = 0$. We call such an estimator an $\emph{ALA tree}$.\footnote{Certain partitions cannot be obtained by recursive binary partitioning. This distinction is not important for our analysis, so we will slightly abuse terminology in calling these estimators trees.}

Note that decision tree algorithms that make non-axis-aligned splits do not yield partitions, though this is not the case for CART and most other algorithms popularly used today. In this definition, we have also kept the partition fixed, whereas decision tree algorithms learn a data-adaptive partition. Having a fixed partition, however, is in keeping with our setting of honest decision trees: We assume that the partition $\partition = \partition(\data[m]')$ has been learnt using a separate dataset $\data[m]'$ that we are conditioning on. Furthermore, we note that any lower bounds that hold conditionally on $\data[m]'$ will also hold unconditionally.

The \emph{squared error risk}, or \emph{generalization error} of an estimator $\hat{f}$ for $f$ is defined as
\begin{equation*}
    \risk(\hat{f}) \coloneqq \E_{\bx \sim \measure}\braces*{\left(\hat f_n (\bx) - f(\bx)\right)^2}.
\end{equation*}
We are interested in the smallest possible risk of an ALA tree. To rule out irregularities that happen when some cell $\cell$ does not contain any samples from the training set $\data$, we need to ensure that the cells are not too small. We say that a partition $\partition$ is \emph{permissible} if for every cell $\cell \in \partition$, we have $\nu\braces{\cell} \geq \frac{1}{n}$. This is a reasonable assumption, as we should expect each cell to contain at least one sample point. Finally, given a conditional mean function $f$, we define the \emph{oracle expected risk} for ALA trees to be 
\begin{equation} \label{eq:definition_of_oracle_risk}
    \oraclerisk(f,\measure,n) \coloneqq \inf_{\partition} \E\braces*{R(\hat{f}(-;\partition,\data)}
\end{equation}
where the infimum is taken over all permissible partitions.

\section{A bias-variance risk decomposition for ALA trees} \label{sec:bias_variance}

Our main results rely on two key ingredients: A bias-variance decomposition of the expected risk for ALA trees, and a connection to information theory. We state the former as follows.

\begin{theorem}[Bias-variance decomposition of expected risk] \label{thm:bias_variance_permissible}
    Assume the regression model \eqref{eq:regression_model}. Given an permissible partition $\partition$ and a training set $\data$, the expected risk satisfies the following lower and upper bounds:
    \begin{equation} \label{eq:simplified_bias_variance_lower_bound}
        \E \risk(\hat f(-; \partition, \data)) \geq \sum_{\cell \in \partition} \Var\braces{f(\bx)~|~\bx \in \cell}\measure\braces{\cell} + \frac{\abs*{\partition}\sigma^2}{2n},
    \end{equation}
    \begin{equation} \label{eq:simplified_bias_variance_upper_bound}
        \E \risk(\hat f(-; \partition, \data)) \leq 7\sum_{\cell \in \partition} \Var\braces{f(\bx)~|~\bx \in \cell}\measure\braces{\cell} + \frac{6\abs*{\partition}\sigma^2}{n} + E(\partition),
    \end{equation}
    where
    \begin{equation*}
        E(\partition) = \sum_{\cell \in \partition} \E\braces{f(\bx)~|~\bx \in \cell}^2 \left(1-\measure\braces{\cell}\right)^n \measure\braces{\cell}.
    \end{equation*}
\end{theorem}

We make a few remarks about the above theorem. First, we draw attention to its generality: It holds for any conditional mean function and any distribution $\measure$ on $\mathcal{X}$, where $\mathcal{X}$ is allowed to be any measurable subset of $\R^d$. In fact, inspecting the proof shows that we do not even require the partition to be axis-aligned.

Next, observe that the lower and upper bounds match up to constant factors and an additive error term $E$ for the upper bound. This term is due to each cell receiving possibly zero samples from the training set, and thus can be made arbitrarily small in comparison with the main terms by further constraining the minimum volume of cells in the partition.

The first main term can be thought of as the approximation error or bias, and has the following equivalent representations:
$$
\sum_{\cell \in \partition} \Var\braces{f(\bx)~|~\bx \in \cell}\measure\braces{\cell} = \E\braces*{\Var\braces*{f(\bx)~|~\bx \in \cell}} = \E\braces*{\paren*{f(\bx)-\bar{f}_{\partition}(\bx)}^2},
$$
where $\bar{f}_{\partition}$ is the function that takes the value of the conditional mean of $f$ over each cell. In other words, this term is the expected mean square error of the ALA tree if we had infinite data.

The second main term is the contribution from variance, and can be traced to using empirical averages over each cell to estimate the conditional means. The form of this term is striking: It scales linearly with the size of the partition, in direct analogy with the penalty term in cost-complexity pruning \citep{friedman2001elements}. Furthermore, it precisely quantifies the trade off between bias and variance when splitting a cell $\cell$ in the partition into two children $\cell_L$ and $\cell_R$. The gain in variance is of the order $\frac{\sigma^2}{n}$, while the reduction in bias is
$$
\Var\braces{f(\bx)~|~\bx \in \cell}\measure\braces{\cell} - \Var\braces{f(\bx)~|~\bx \in \cell_L}\measure\braces{\cell_L} - \Var\braces{f(\bx)~|~\bx \in \cell_R}\measure\braces{\cell_R}.
$$
One can check that this is the population version of the weighted impurity decrease of this split, which is the quantity used to determine splits in CART, and also the value compared against a threshold in early stopping with the minimum impurity decrease criterion. These observations show that both these methods for preventing overfitting in CART attempt to optimize an objective function that is a weighted combination of plug-in estimates of the bias and variance terms in the expected risk decomposition for an honest tree.

The proof of Theorem \ref{thm:bias_variance_permissible} and that of a tighter but more complicated version of the decomposition can both be found in Appendix \ref{sec:proof_of_bias_variance}. In the next two sections, we will see how the decomposition can be used in conjunction with rate distortion theory to yield lower bounds for additive models.

\section{A connection of decision tree estimation to rate-distortion theory} \label{sec:rate_distortion}


The second ingredient we need is supplied by rate-distortion theory. We start by recalling some definitions from \cite{cover2012elements}. We will use $H(-)$, $h(-)$ and $I(-;-)$ to denote discrete entropy, differential entropy, and mutual information respectively.
Let $\mathcal{X}$ be a subset of $\R^\dim$ as before. Given a vector $\bbeta \in \R^d$, we denote the associated weighted Euclidean norm on $\mathcal{X}$ via $\norm{\bx-\by}_{\bbeta}^2 \coloneqq \sum_{j=1}^d \beta_j^2 \paren{x_j-y_j}^2$. Now let $p$ denote a joint distribution on $\mathcal{X}\times\mathcal{X}$. The \emph{distortion} of $p$ with respect to $\norm{-}_{\bbeta}$ is defined as
$$
\delta(p;\bbeta) \coloneqq \E_{(\bx,\hat{\bx}) \sim p}\braces*{\norm*{\bx-\hat\bx}_{\bbeta}^2}.
$$
The \emph{rate distortion function} of the marginal $p_{\bx}$ is defined by
$$
R(D;p_{\bx},\bbeta) \coloneqq \inf_{p_{\hat \bx | \bx}} I(\bx;\hat\bx)
$$
where the infimum is taken over all conditional distributions such that $\delta(p_\bx p_{\hat \bx |\bx};\bbeta) \leq D$.

In rate-distortion theory, the rate distortion function characterizes the length of a binary code needed to encode a source so that the distortion is not too large. In this paper, it clarifies the trade-off between the bias and variance terms in the decomposition  \eqref{eq:simplified_bias_variance_lower_bound}. Under some independence conditions, we show that the bias term is equivalent to a distortion, while the size of the partition occurring in the variance term is bounded from below by the rate of this distortion. More precisely, we have the following lemma.

\begin{lemma}[Rate-distortion bound for oracle expected risk] \label{lem:rate_distortion_lower_bound}
    Assume the regression model \eqref{eq:regression_model}, and that $\mathcal{X} = \hypercube$ or $\mathcal{X} = [0,1]^\dim$. Furthermore, assume that the covariates are independent, and that the conditional mean function is linear: $f(\bx) = \bbeta^T \bx$. Then the oracle expected risk is lower bounded by
    \begin{equation} \label{eq:rate_distortion_lower_bound_for_linear}
        \oraclerisk(f,\nu,n) \geq \frac{1}{2}\inf_{D > 0} \braces*{D + \frac{\sigma^2 2^{R(D;\measure,\bbeta)}}{n}}.
    \end{equation}
\end{lemma}

\begin{proof}
    Consider some permissible partition $\partition$. For any cell $\cell \in \partition$, notice that the conditional covariate distribution $\measure|_{\cell}$ also has independent covariates. Let $\bx'$ be an independent copy of $\bx$. Using independence, we compute
    \begin{align} \label{eq:lower_bound_conditional_variane_rate}
        \Var\braces*{f(\bx)~|~\bx \in \cell} & = \frac{1}{2}\E\braces*{\paren*{\bbeta^T\paren*{\bx - \bx'}}^2~|~\bx,\bx' \in \cell} \nonumber\\
        & = \frac{1}{2}\E\braces*{\sum_{j=1}^d \beta_j^2 \paren*{x_j - x_j'}^2~|~\bx,\bx' \in \cell} \nonumber\\
        & = \frac{1}{2}\E\braces*{\norm{\bx - \bx'}_{\bbeta}^2~|~\bx,\bx' \in \cell} \nonumber\\
        & \geq \frac{1}{2}\E\braces*{\norm{\bx - \bz(\cell)}_{\bbeta}^2~|~\bx \in \cell},
    \end{align}
    where $\bz(\cell) \coloneqq \arg\min_{\bx' \in \cell} \E\braces*{\norm{\bx - \bx'}_{\bbeta}^2~|~\bx \in \cell}$.\footnote{$\bz(\cell)$ is the cell centroid when $\mathcal{X}$ is the unit length cube, but not when it is the Boolean cube. Furthermore, equality actually holds without the factor of $1/2$ in the former case.}
    To define a conditional distribution, for each $\bx$, we let $p_{\hat\bx|\bx}$ be a Dirac mass at $\bz(\cell(\bx))$, where $\cell(\bx)$ is the cell in $\partition$ containing $\bx$. Then the bias term in \eqref{eq:simplified_bias_variance_lower_bound} can be lower bounded by the distortion for the joint distribution $p = \nu p_{\hat\bx|\bx}$:
    $$
    \sum_{\cell \in \partition} \Var\braces{f(\bx)~|~\bx \in \cell}\measure\braces{\cell} \geq \frac{\delta(p;\bbeta)}{2}.
    $$
    Meanwhile, notice that $\hat\bx$ is a discrete distribution on $\abs*{\partition}$ elements, so we may use the max entropy property of the uniform distribution to write
    $$
    \log \abs*{\partition} \geq H(\hat{\bx}) \geq I(\bx;\hat\bx) \geq R(\delta(p;\bbeta);\measure,\bbeta).
    $$
    Plugging these formulas into \eqref{eq:simplified_bias_variance_lower_bound} gives the lower bound
    $$
    \sum_{\cell \in \partition} \Var\braces{f(\bx)~|~\bx \in \cell}\measure\braces{\cell} + \frac{\abs*{\partition}\sigma^2}{2n} \geq \frac{1}{2}\paren*{\delta(p;\bbeta) + \frac{\sigma^2 2^{R(\delta(p;\bbeta);\measure,\bbeta)}}{n}}.
    $$
    Minimizing over all partitions yields \eqref{eq:rate_distortion_lower_bound_for_linear}.
\end{proof}

We remark that the lemma applies to both continuous and discrete distribution, and may be valid for other subsets $\mathcal{X} \subset \R^d$. Next, it is known that rate distortion functions are convex and monotonically decreasing, and one may therefore check that the same applies to the function $D \mapsto 2^{R(D;\nu,\bbeta)}$. As such, the right-hand-side of \eqref{eq:rate_distortion_lower_bound_for_linear} is the solution to a convex optimization problem, and can be solved efficiently in principle. This is especially significant when $\mathcal{X}$ is the Boolean cube, because it allows us to turn what is a priori a combinatorial optimization problem into a smooth, convex one.

When $\bbeta$ has at least $s$ large coefficients, the independence of the covariates allows us to use standard calculations to bound $R(D;\nu,\bbeta)$ from below by elementary functions, therefore giving us a closed-form formula for the oracle expected risk. We state the result for the case where $\mathcal{X}$ is the unit length cube, and leave more general versions of the theorem to the next section. 

\begin{theorem}[Lower bound for linear models] \label{thm:linear_model_lower_bound}
    In addition to the assumptions of Lemma \ref{lem:rate_distortion_lower_bound}, assume that $\mathcal{X}$ is the unit length cube $[0,1]^\dim$, and that the covariates each follow some marginal distribution $\nu_0$. Suppose there is some subset of coordinates $S \subset \coordindices$ of size $s$ such that $\abs*{\beta_j} \geq \beta_0$ for all $j \in S$. Then the oracle expected risk is lower bounded by
    \begin{equation} \label{eq:linear_model_lower_bound}
        \oraclerisk(f,\nu,n) \geq s2^{\frac{2s}{s+2}h(\nu_0)-2}\paren*{\frac{\beta_0^2}{\pi e}}^{s/(s+2)}\paren*{\frac{\sigma^2}{n}}^{2/(s+2)}.
    \end{equation}
\end{theorem}

\begin{proof}
    Combining Lemmas \ref{lem:rate_for_product}, \ref{lem:rate_for_dominated_norm}, and \ref{lem:rate_for_univariate}, we know that $R(D;\nu,\bbeta)$ is lower bounded by the value of
    \begin{equation} \label{eq:rate_lower_bound_for_cts_linear_opt}
        \inf_{\beta_0^2 \sum_{j \in S} D_j \leq D} \sum_{j\in S} \paren*{ h(\nu_0) - \frac{1}{2}\log\paren*{2\pi e D_j}}\vee 0.
    \end{equation}
    
    This is a convex optimization problem, and by symmetry over the coordinate indices in $S$, it is easy to see that the infimum is achieved at $D_j = \frac{D}{s \beta_0^2}\vee \frac{2^{2h(\nu_0)}}{2\pi e}$ for $j \in S$, and $D_j = \frac{2^{2h(\nu_0)}}{2\pi e}$ for $j \notin S$, where $s = \abs{S}$. Plugging these into \eqref{eq:rate_lower_bound_for_cts_linear_opt}, we get
    \begin{equation} \label{eq:rate_distortion_final_bound_for_linear_cts}
        R(D;\nu,\bbeta) \geq s \paren*{h(\nu_0) - \frac{1}{2} \log\paren*{\frac{2\pi e D}{s\beta_0^2}}}.
    \end{equation}
    
    As such, we have
    $$
    D + \frac{\sigma^2 2^{R(D;\measure,\bbeta)}}{n} \geq D + \frac{\sigma^2 2^{h(\nu_0)}}{n} \paren*{\frac{s\beta_0^2}{2\pi e D}}^{s/2}.
    $$
    Differentiating, we easily see that the minimum is achieved at
    $$
    D = s2^{\frac{2s}{s+2}h(\nu_0)-1}\paren*{\frac{\beta_0^2}{\pi e}}^{s/(s+2)}\paren*{\frac{\sigma^2}{n}}^{2/(s+2)}.
    $$
    Using this value in \eqref{eq:rate_distortion_lower_bound_for_linear} and dropping the second term then completes the proof.
\end{proof}


\section{Results for additive models} \label{sec:additive_model_bounds}

In the previous section, we saw how the bias-variance risk decomposition and an information theoretic argument can be used to obtain a lower bound for oracle expected risk. With a more sophisticated application of the latter, we can derive more powerful results for additive models over both continuous and Boolean feature spaces. We state these results in this section, deferring all proofs to the appendix because of space constraints.

\begin{theorem}[Lower bound for additive models on unit length cube] \label{thm:lower_bound_additive_models}
    Assume the regression model \eqref{eq:regression_model}, with $f$ be defined as in \eqref{eq:additive_model}, and assume that the covariate space is the unit length cube $[0,1]^\dim$. Suppose $\phi_j \in C^1([0,1])$ for $j=1,\ldots,\dim$. Let $I_1, I_2,\ldots, I_d \subset [0,1]$ be sub-intervals, and suppose there is some subset of indices $S \subset \coordindices$ of size $s$ such that $\displaystyle \min_{t \in I_j} \abs*{\phi_{j}'(t)} \geq \beta_0 > 0$ for all $j \in S$. Denote $\mathcal{K} = \braces*{\bx \colon x_j \in I_j ~\text{for}~j=1,\ldots,d}$. Assume that $\nu$ is a continuous distribution with density $q$, and denote $q_{min} = \displaystyle\min_{\bx \in \mathcal{K}} q(\bx)$. Then the oracle expected risk is lower bounded by
    \begin{equation} \label{eq:additive_model_lower_bound_specialized}
        \oraclerisk(f,\measure,n) \geq s\mu(\mathcal K) \paren*{\frac{\beta_0^2q_{min}}{12}}^{s/(s+2)} \paren*{\frac{\sigma^2}{4n}}^{2/(s+2)}.
    \end{equation}
\end{theorem}

As mentioned before, the $\Omega(n^{-2/(s+2)})$ rate in \eqref{eq:additive_model_lower_bound_specialized} is the $\ell_2$ minimax rate for nonparametric estimation of $C^1$ functions in $s$ dimensions \citep{stone1982optimal}. This is far worse than the minimax rate for estimating sparse additive models, which scales as $\max\braces*{\frac{s\log(d/s)}{n},s\epsilon_n^2(\mathcal{H})}$, where $\epsilon_n(\mathcal{H})$ is a quantity that depends only on $\mathcal{H}$ and the sample size $n$ \citep{raskutti2012minimax}.

The theorem is more flexible than Theorem \ref{thm:linear_model_lower_bound} in the following ways: It allows the component functions $\phi_j$ to be nonlinear, and even have vanishing derivatives everywhere except on an interval, which means \eqref{eq:additive_model_lower_bound_specialized} applies to any nontrivial choice of the $\phi_j$'s. Furthermore, unlike Theorem \ref{thm:linear_model_lower_bound}, it does not require the covariates to be independent. Finally, a more general version of the lower bound, stated as Theorem \ref{thm:lower_bound_additive_models_general} in Appendix \ref{sec:covering_proof}, allows us to provide tighter lower bounds in the case where the $\beta_j$'s may be decaying in magnitude rather than having a non-zero lower bound.

These improvements require a different information theoretic argument, which roughly works as follows: First, we derive the maximum volume a cell can have under gradient lower bounds and a prescribed variance constraint (see Lemma \ref{lem:variance_and_volume}.) We then use this to compute the number of cells necessary to cover the portion of $[0,1]^\dim$ over which the the gradient lower bounds hold. As an easy by-product of the above calculations, we also compute the optimal dimensions of a cell under the variance constraint. This allows us to derive matching oracle upper bounds for sparse additive models: 

\begin{proposition}[Upper bound for sparse additive models on unit length cube] \label{prop:sparse_additive_upper_bound}
    Let $f$ be a sparse additive model, i.e. there is a subset of coordinates $S \subset \coordindices$ such that $f(\bx) = \sum_{j \in S} \phi_j(x_j)$. Assume that the covariate space is the unit length cube $[0,1]^\dim$. Suppose $\phi_j \in C^1([0,1])$, and $\norm{\phi_j}_\infty \leq \beta_{max}$ for $j \in S$. Assume $\nu$ is a continuous distribution with density $q$. Then
    \begin{equation}
        \limsup_{n \to \infty}\frac{\oraclerisk(f,\measure,n)}{n^{-2/(s+2)}} \leq \paren*{168s\norm{q}_\infty\beta_{max}^2 + 6}\sigma^{2s/(s+2)}.
    \end{equation}
\end{proposition}

Our next main result is for additive models over the Boolean cube. Note that all additive models are linear in this setting, and we are thus able to prove this using the original rate-distortion argument.

\begin{theorem}[Lower bounds for additive Boolean models]
\label{thm:lower_bound_boolean_additive_model}
    Assume the regression model \eqref{eq:regression_model} and that the conditional mean function is linear: $f(\bx) = \bbeta^{T}\bx$. Assume that the covariate space is the hypercube $\hypercube$, and that the covariates are independent, with $x_j \sim \text{Ber}(\pi)$, $0 \leq \pi \leq \frac{1}{2}$, for $j=1,\ldots,d$. Suppose there is some subset of coordinates $S \subset \coordindices$ of size $s$ such that $\abs*{\beta_j} \geq \beta_0 > 0$ for all $j \in S$. Then the oracle expected risk is lower bounded by
    \begin{equation} \label{eq:risk_boolean_lower_bound_specialized}
        \oraclerisk(f,\nu,n) \geq \frac{s\beta^2_0}{2}\paren*{1 - \paren*{\frac{2e^sn\beta^2_0}{2^{sH(\pi)}\sigma^2}}^{\frac{1}{s-1}}}.
    \end{equation}
\end{theorem}

The form of the lower bound \eqref{eq:risk_boolean_lower_bound_specialized} is different from that in Theorem \ref{thm:lower_bound_additive_models}. This is due to the fact that we can achieve zero approximation error over the Boolean cube with finitely many cells. As a consequence, while the rate-distortion function in the continuous case \eqref{eq:rate_distortion_final_bound_for_linear_cts} tends to infinity as $D$ tends to 0, that in the Boolean case \eqref{eq:rate_distortion_final_bound_for_linear_boolean} tends to a finite number. Our proof of \eqref{eq:risk_boolean_lower_bound_specialized} does not actually use the sharp rate bound \eqref{eq:rate_distortion_final_bound_for_linear_boolean}, and instead approximates it by a more computationally tractable bound \eqref{eq:loose_rate_bound_boolean}. It is unclear how much slack from this approximation propagates into the final bound \eqref{eq:risk_boolean_lower_bound_specialized}, but it is reassuring that the general concave down shape of the test error scaling in Figure \ref{fig:linear_boolean} is consistent with \eqref{eq:risk_boolean_lower_bound_specialized}. We provide a more general version of the lower bound (Theorem \ref{thm:lower_bound_boolean_additive_model_general}) in Appendix \ref{sec:rate_distortion_proofs}.

There are a few interesting additional observations that can be made. First, we remark that although Theorem \ref{thm:lower_bound_additive_models} is stated in terms of distributions over the unit cube, we can easily extend it to non-compact distributions over $\R^\dim$ such as multivariate Gaussians by using marginal quantile transforms.\footnote{To see this, let $F_j$ denote the CDF for the marginal distribution for $x_j$. Then writing $\tilde x_j \coloneqq F_j(x_j)$ for $j=1,\ldots,d$, the random vector $\tilde \bx = (\tilde x_1,\ldots,\tilde x_\dim)$ takes values in $[0,1]^\dim$, and we have $f(\bx) = \sum_{j=1}^\dim \phi_j\paren*{F_j^{-1}(\tilde x_j)}$, which, as an additive model with respect to $\tilde \bx$, now satisfies the hypotheses of the theorem.} Second, from the formulas \eqref{eq:linear_model_lower_bound} and \eqref{eq:risk_boolean_lower_bound_specialized}, we see that the lower bound decreases with the entropy of the covariate distribution, although the rate in the sample size $n$ remains the same. Third, in Appendix \ref{sec:rate_distortion_proofs}, we provide more complicated lower bounds \eqref{eq:general_additive_model_lower_bound} and \eqref{eq:risk_boolean_lower_bound}, which, for a fixed value of $\norm{\bbeta}_2^2$, are smaller when more of the $\ell_2$ energy is concentrated in fewer coordinates, i.e. when the coefficients experience faster decay. This agrees with our intuition that decision trees are adaptive to low-dimensional structure beyond the hard sparsity regime, which has been the focus of recent literature \citep{syrgkanis2020estimation,klusowski2020,klusowski2021universal}. 

\section{Numerical simulations}

We examined the empirical validity of each of our main results by simulating the generalization error of tree-based algorithms fitted to sparse linear models with both continuous and Boolean features, as well as an additive non-linear sum of squares model with continuous features. Details of our experimental design and algorithm settings are given in the following paragraphs. 

\noindent\textbf{Experimental design:} For Figures \ref{fig:linear_cts} and \ref{fig:linear_boolean}, we simulate data from a sparse linear generative model $y = \bbeta^T\bx + \epsilon$ with $\bx \sim \text{Unif}\paren*{[0,1]^\dim}$ and $\bx \sim \text{Unif}\paren*{\{0,1\}^\dim}$ respectively. In Figure \ref{fig:sum_of_squares_cts}, we simulate data via a sparse sum of squares model $y = \sum_j \beta_j x^2_j + \epsilon$  with $\bx \sim \text{Unif}\paren*{[0,1]^\dim}$. In all of the experiments, we varied $n$, but fixed $d = 50$, $\sigma^2 = 0.01$, and set $\beta_j = 1$ for $j=1,\ldots,s$, and $\beta_j=0$ otherwise, where $s$ is a sparsity parameter. We ran the experiments with both $s=10$ and $s=20$, and plotted the results for each setting in panel A and panel B respectively for all of the figures. We computed the generalization error using a test set of size 500, averaging the results over 25 runs.

\noindent\textbf{Algorithm settings:} In all of our experiments, we fit both honest and non-honest versions of CART, as well as the non-honest version of RF using a training set of size $n$. For the honest version of CART, we use a separate independent sample of size $n$ to compute averages over each leaf in the tree. Furthermore, if a cell contains no samples from the training data used to do averaging, we search for the closest ancestor node that contains at least one sample and use the average over that node to make a prediction. We use \texttt{min\_samples\_leaf=5} as the stopping condition, although we also ran experiments with cost-complexity pruning and achieved very similar results.

We note that Figures   \ref{fig:linear_cts} and \ref{fig:sum_of_squares_cts}  not only give an empirical validation of the theoretical rates in our lower bound, which are $0.17$ and $0.09$ for $s=10 \text{ and } 20$ respectively, but also indicate that honest CART almost achieves these bounds despite there being no a priori guarantee that CART grows an optimal tree. While the theory does not cover the case of non-honest CART, its test error is worse than that of honest CART, and has a similar rate. Furthermore in Figure \ref{fig:linear_boolean}, we see that the general concave down shape of the test error scaling is consistent with the theoretical bound \eqref{eq:risk_boolean_lower_bound_specialized}. An interesting facet of all our simulations is that RF has a markedly faster rate, implying that diverse trees allow the algorithm to pool information across the training samples more efficiently, supporting \cite{breiman2001random}'s original hypothesis.

\begin{figure}[H]
    \centering
    \includegraphics[width=0.9\textwidth]{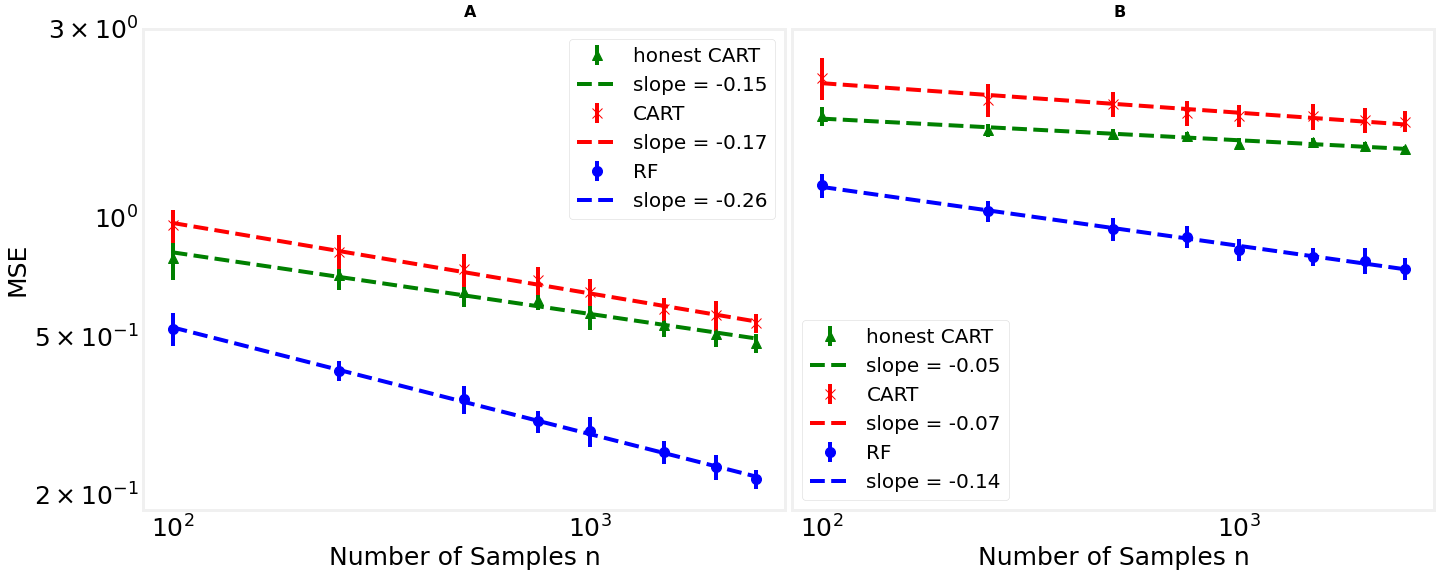}
    \caption{Scaling of the test set error for CART and RF for a sparse linear generative model $y = \bbeta^T\bx + \epsilon$ with $\bx \sim \text{Unif}\paren*{[0,1]^\dim}$.We show the scaling with respect to $n$ for \textbf{(A)} $s=10$, and \textbf{(B)} $s=20$.}
    \label{fig:linear_cts}
\end{figure}

\begin{figure}[H]
    \centering
    \includegraphics[width=0.9\textwidth]{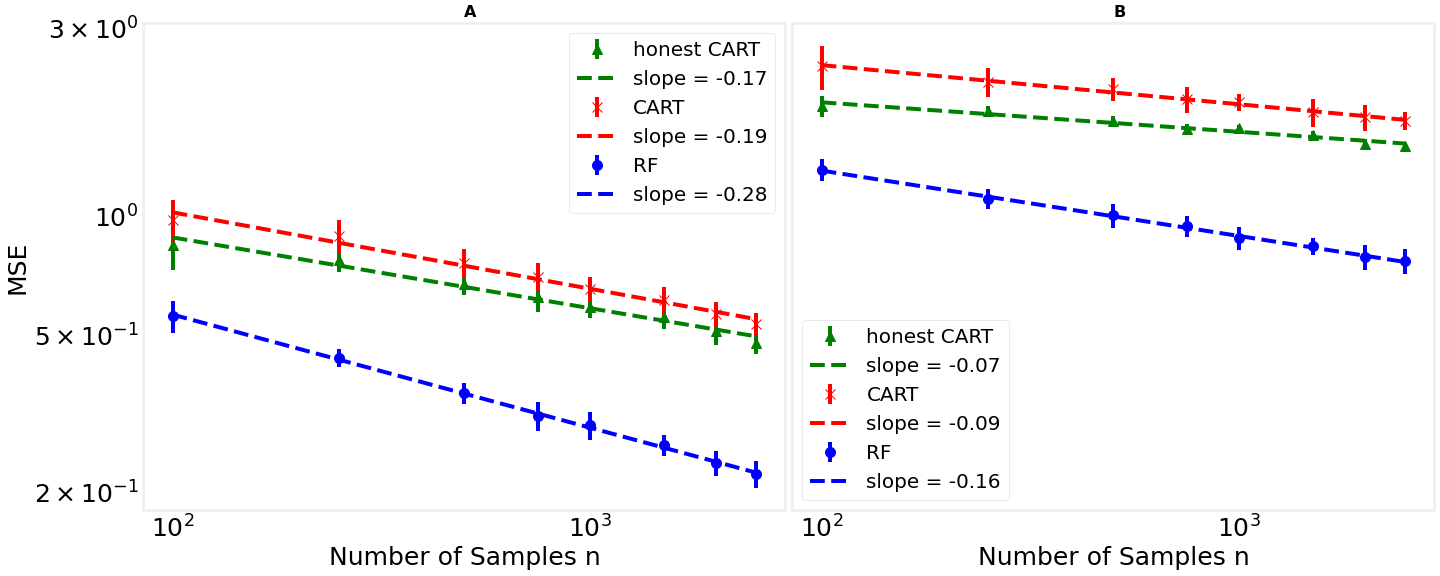}
    \caption{Scaling of the test set error for CART and RF for a sparse sum of squares generative model $y = \sum_j \beta_j x^2_j + \epsilon$ with $\bx \sim \text{Unif}\paren*{[0,1]^\dim}$.We show the scaling with respect to $n$ for \textbf{(A)} $s=10$, and \textbf{(B)} $s=20$.}
    \label{fig:sum_of_squares_cts}
\end{figure}

\begin{figure}[h]
    \centering
    \includegraphics[width=0.9\textwidth]{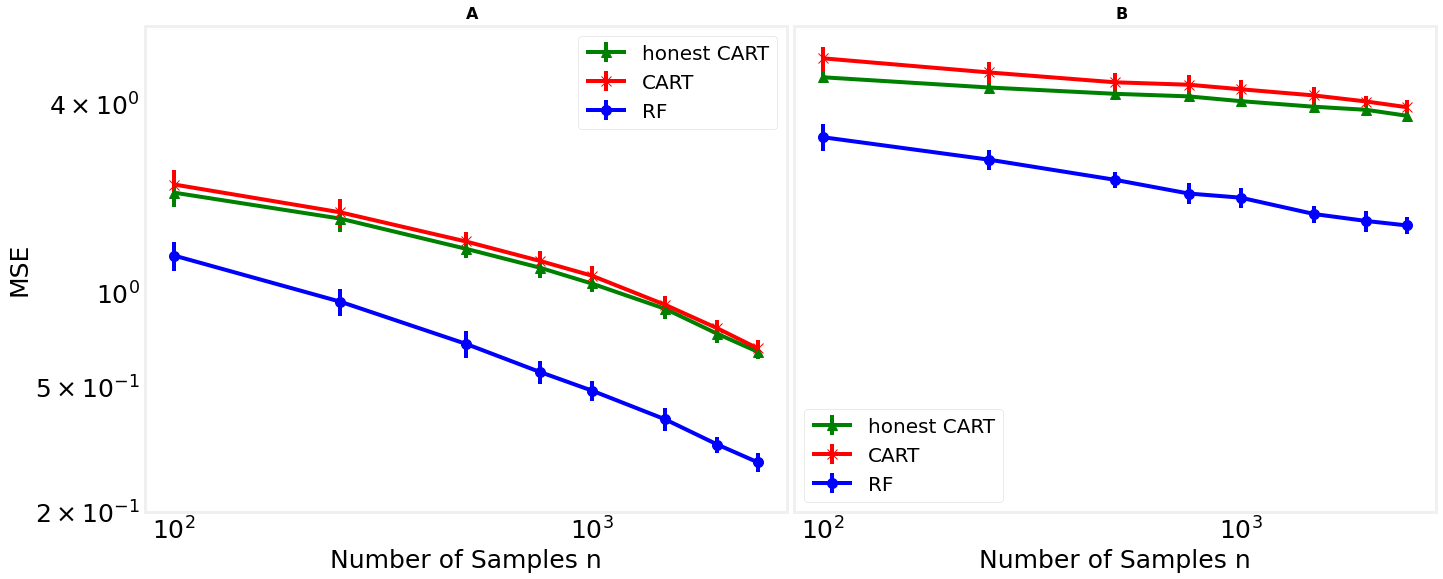}
    \caption{Scaling of the test set error for CART and RF for a sparse linear generative model $y = \bbeta^T\bx + \epsilon$ with $\bx \sim \{0,1\}^\dim$ and each $x_{j} \sim \text{Ber}(\frac{1}{2})$.We show the scaling with respect to $n$ for \textbf{(A)} $s=10$, and \textbf{(B)} $s=20$.}
    \label{fig:linear_boolean}
\end{figure}

\section{Discussion}

In this paper, we have obtained theoretical lower bounds on the expected risk for honest ALA trees when fitted to additive models, while our simulations suggest that these results should also hold for their non-honest counterparts. These bounds lead us to argue that such estimators, including CART, have an inductive bias against global structure, a bias that arises not from the greedy splitting criterion used by most decision tree algorithms, but from the leaf-only averaging property of this class of estimators. Furthermore, we provide experimental evidence that the bounds do not apply to RF, which supports \cite{breiman2001random}'s original narrative that the diversity of trees in a forest helps to reduce variance and improve prediction performance. Nonetheless, the rates exhibited by RF are still significantly slower than the minimax rates for sparse additive models, hinting at fundamental limits we are yet to understand.

Our results further the conversation about how decision tree algorithms can be improved, and suggest that they should be modified to more easily learn global structure. One natural idea on how to do this is to adopt some type of hierarchical shrinkage or global pooling. Another is to combine tree-based methods with linear or additive methods in a way that incorporates the statistical advantages of both classes of methods, in the vein of \cite{friedman2008predictive}'s RuleFit.\footnote{Given the interest of practitioners in using tree-based methods to identify interactions in genomics \citep{chen2012random,boulesteix2012overview,basu2018iterative,behr2021provable}, it is fair to say that this is a key strength of trees and RF.} Recently, \cite{bloniarz2016supervised} and \cite{friedberg2020local} suggested using the RF kernel in conjunction with local linear (or polynomial) regression, while \cite{kunzel2019linear} replaced the constant prediction over leaf with a linear model. These works, however, aim at modifying RF to better exploit smoothness, and do not directly address the loss in power for detecting global structure that comes from partitioning the covariate space. Furthermore, the focus on forests forestalls the possibility of preserving interpretability.


Taking a step back to look at the bigger picture, we believe that it is important to analyze the generalization performance of CART and other decision tree algorithms on other generative regression models in order to further elicit their inductive biases. The same type of analyses can also be applied to other machine learning algorithms. Since real world data sets often present some structure that can be exploited using the right inductive bias, this research agenda will allow us to better identify which algorithm to use in a given application, especially in settings, such as the estimation of heterogeneous treatment effects, where a held out test set is not available. Moreover, as seen in this paper, such investigations can yield inspiration for improving existing algorithms.

The approach we follow is different from the classical paradigm of statistical estimation, which starts with an estimation problem, and then searches for estimation procedures that can achieve some form of optimality. Instead, as is common in machine learning, we take an algorithm as the primitive object of investigation, and seek to analyze its performance under different generative models in order to elicit its inductive bias. This approach is more aligned with modern data analysis, in which we seldom have a good grasp over the functional form of the data generating process, leading to a handful of general purpose algorithms being used for the vast majority of prediction problems. This embrace of suboptimality is consistent with viewing the models as approximations -- an old tradition in the statistical literature (\cite{huber1967behavior,box1979robustness,grenander1981abstract,geman1982nonparametric,buja2019models}).

We have only scratched the surface of investigating the inductive biases of decision trees, RF, gradient boosting, and other tree-based methods, and envision an abundant garden for future work.


\subsection*{Acknowledgements}

This research is kindly supported in part by NSF TRIPODS Grant 1740855, DMS-1613002, 1953191, 2015341, IIS 1741340,  the Center for Science of Information (CSoI), an NSF Science and Technology Center, under grant agreement CCF-0939370, NSF grant 2023505 on Collaborative Research: Foundations of Data Science Institute (FODSI), the NSF and the Simons Foundation for the Collaboration on the Theoretical Foundations of Deep Learning through awards DMS-2031883 and 814639, and a Chan Zuckerberg Biohub Intercampus Research Award.

\bibliography{rf}
\bibliographystyle{plainnat}

\appendix

\section{Proof of Theorem \ref{thm:bias_variance_permissible}} \label{sec:proof_of_bias_variance}

We first state and prove a tighter and more complicated version of the bias-variance decomposition of the expected risk. For notational convenience, we denote
$$
\hat\E_{\cell}\braces{y} \coloneqq \hat\E\braces{y|\bx \in \cell} = \frac{1}{N(\cell)}\sum_{\bx^{(i)} \in \cell} y^{(i)}.
$$

\begin{proposition}[Bias-variance decomposition of risk] \label{prop:bias_variance_risk}
    Assume the regression model \eqref{eq:regression_model}. Given a partition $\partition$ and a training set $\data$, the expected squared error risk satisfies the following upper and lower bounds
    \begin{equation} \label{eq:bias_variance_lower_bound}
        \E \risk(\hat f(-; \partition, \data)) \geq \sum_{\cell \in \partition}\Var\braces{f(\bx)~|~\bx \in \cell}\left(\measure\braces{\cell}+\frac{1}{n}\right) + \frac{|\partition|\sigma^2}{n} + E_1 - E_2
    \end{equation}
    \begin{equation} \label{eq:bias_variance_upper_bound}
        \E \risk(\hat f(-; \partition, \data)) \leq \sum_{\cell \in \partition}\Var\braces{f(\bx)~|~\bx \in \cell}\left(\measure\braces{\cell}+\frac{6}{n}\right) + \frac{6|\partition|\sigma^2}{n} + E_1
    \end{equation}
    where
    \begin{equation*}
        E_1 = \sum_{\cell \in \partition} \E\braces{f(\bx)~|~\bx \in \cell}^2 \left(1-\measure\braces{\cell}\right)^n \measure\braces{\cell}.
    \end{equation*}
    \begin{equation*}
        E_2 = \frac{1}{n}\sum_{\cell \in \partition} \left(\Var\braces{f(\bx)~|~\bx \in \cell} + \sigma^2\right) \left(1-\measure\braces{\cell}\right)^n.
    \end{equation*}
\end{proposition}

\begin{proof}
    First consider a cell $\cell$, with $\bx \in \cell$. Supposing that $N(\cell)\neq0$, We have
    \begin{align*}
        \E_{\data}\braces*{\left(f(\bx) - \hat\E_{\cell}\braces*{y} \right)^2~|~ N(\cell)} & = \E_{\data}\braces*{\left(f(\bx) - \frac{1}{N(\cell)} \sum_{\bx^{(i)} \in \cell} \left(m(\bx^{(i)}) + \epsilon_i \right) \right)^2~\vline~ N(\cell)} \\
        & = \E_{\data}\braces*{\left(f(\bx) - \frac{1}{N(\cell)} \sum_{\bx^{(i)} \in \cell} m(\bx^{(i)})\right)^2~\vline~ N(\cell)} + \frac{\sigma^2}{N(\cell)}.
    \end{align*}
    Taking a further conditional expectation with respect to $\bx \in \cell$, we see that the distribution $\bx$ is the same as that of each $\bx^{(i)}$. We can therefore compute
    $$
    \E_\bx\braces*{\E_{\data}\braces*{\left(f(\bx) - \frac{1}{N(\cell)} \sum_{\bx^{(i)} \in \cell} m(\bx^{(i)})\right)^2~\vline~ N(\cell)} ~\vline~ \bx \in \cell } = \left(1 + \frac{1}{N(\cell)}\right)\Var\braces{f(\bx) ~|~\bx \in \cell}.
    $$
    Putting these two calculations together, and interchanging the order of expectation, we have
    \begin{align*}
        \E_{\data}\braces*{\E_{\bx}\braces*{\left(f(\bx) - \hat\E_{\cell}\braces*{y} \right)^2 ~|~ \bx \in \cell}~|~ N(\cell)} & = \E_{\bx}\braces*{\E_{\data}\braces*{\left(f(\bx) - \hat\E_{\cell}\braces*{y} \right)^2~|~ N(\cell)} ~|~ \bx \in \cell} \\
        & = \Var\braces{f(\bx) ~|~\bx \in \cell} + \frac{\Var\braces{f(\bx) ~|~\bx \in \cell} + \sigma^2}{N(\cell)}.
    \end{align*}
    
    Recall our convention that we set $\hat\E_{\cell}\braces*{y} = 0$ if $N(\cell)=0$. We may then write
    \begin{align} \label{eq:decomposing_MSE}
        \E_{\data,\bx}\braces*{\left(f(\bx) - \hat f(\bx) \right)^2~\vline~ \bx \in \cell} & = \E_{\data,\bx}\braces*{\left(f(\bx) - \hat f(\bx) \right)^2\indicator\braces{N(\cell)\neq0}~\vline~ \bx \in \cell} \nonumber \\
        & \quad\quad + \E_{\data,\bx}\braces*{\left(f(\bx) - \hat f(\bx) \right)^2\indicator\braces{N(\cell)=0}~\vline~ \bx \in \cell} \nonumber\\
        & = \E_{\data}\braces*{\E_{\data}\braces*{\E_{\bx}\braces*{\left(f(\bx) - \hat\E_{\cell}\braces*{y} \right)^2 ~\vline~ \bx \in \cell}~\vline~ N(\cell)}\indicator\braces{N(\cell)\neq0}} \nonumber\\
        & \quad\quad + \E\braces*{f(\bx)^2~\vline~ \bx \in \cell}\P\braces*{N(\cell)=0} \nonumber \\
        & = \Var\braces{f(\bx) ~|~\bx \in \cell}\P\braces*{N(\cell)\neq0} + \E\braces*{f(\bx)^2~\vline~ \bx \in \cell}\P\braces*{N(\cell)=0} \nonumber\\
        & \quad\quad + \left(\Var\braces{f(\bx) ~|~\bx \in \cell} + \sigma^2\right)\E_{\data}\braces*{\frac{\indicator\braces*{N(\cell)\neq0}}{N(\cell)}}
    \end{align}
    
    Note that
    $$
    \E\braces*{f(\bx)^2~\vline~ \bx \in \cell} = \Var\braces{f(\bx) ~|~\bx \in \cell} + \E\braces*{f(\bx)~\vline~ \bx \in \cell}^2
    $$
    so that the first two terms on the right hand side of \eqref{eq:decomposing_MSE} can be rewritten as
    \begin{align*}
        \Var\braces{f(\bx) ~|~\bx \in \cell} + \E\braces*{f(\bx)~\vline~ \bx \in \cell}^2\P\braces*{N(\cell)=0}
    \end{align*}
    
    We continue by providing upper and lower bounds for $\E_{\data}\braces*{\frac{\indicator\braces*{N(\cell)\neq0}}{N(\cell)}}$. For the upper bound, recall that $N(\cell)$ is a binomial random variable, and so has variance smaller than its expectation. This allows to apply Chebyshev's inequality to get
    \begin{align*}
        \P\braces*{\left| N(\cell) - \E N(\cell) \right| \geq \frac{\E N(\cell)}{2}} & \leq \frac{\E N(\cell)}{\left(\frac{1}{2}\E N(\cell)\right)^2} \\
        & = \frac{4}{\E N(\cell)}
    \end{align*}
    Next, since $\frac{\indicator\braces*{N(\cell)\neq0}}{N(\cell)} \leq 1$ we have
    \begin{align*}
        \E\braces*{\frac{\indicator\braces*{N(\cell)\neq0}}{N(\cell)}} & \leq \frac{2}{\E N(\cell)}\P\braces*{N(\cell) \geq \frac{\E N(\cell)}{2}} + \P\braces*{N(\cell) \leq \frac{\E N(\cell)}{2}} \\
        & \leq \frac{2}{\E N(\cell)} + \frac{4}{\E N(\cell)} \\
        & = \frac{6}{\E N(\cell)}.
    \end{align*}
    The right hand side of \eqref{eq:decomposing_MSE} is bounded above by
    $$
    \Var\braces{f(\bx) ~|~\bx \in \cell} + \E\braces*{f(\bx)~\vline~ \bx \in \cell}^2\P\braces*{N(\cell)=0} + \frac{6}{\E N(\cell)}\left(\Var\braces{f(\bx) ~|~\bx \in \cell} + \sigma^2\right)
    $$
    
    Observe that $\P\braces*{N(\cell)=0} = (1-\nu\braces{\cell})^n$, and $\E N(\cell) = n\nu\braces{\cell}$. Taking expectation with respect to $\bx$ then gives
    \begin{align*}
            \E_{\data,\bx}\braces*{\left(f(\bx) - \hat f(\bx) \right)^2} & \leq \sum_{\cell \in \partition}\Var\braces{f(\bx) ~|~\bx \in \cell}\nu\braces{\cell} + \sum_{\cell \in \partition} \E\braces*{f(\bx)~\vline~ \bx \in \cell}^2 (1-\nu\braces{\cell})^n\nu\braces{\cell} \\
    & \quad\quad +  \frac{6}{n}\sum_{\cell \in \partition}  \left(\Var\braces{f(\bx) ~|~\bx \in \cell} + \sigma^2\right).
    \end{align*}
    Rearranging this gives \eqref{eq:bias_variance_upper_bound}.
    
    The lower bound is follows from Cauchy-Schwarz. We have
    \begin{align*}
        \P\braces*{N(\cell)\neq0} & = \E\braces*{\frac{N(\cell)^{1/2}\indicator\braces*{N(\cell)\neq0}}{N(\cell)^{1/2}}} \\
        & \leq \E\braces*{N(\cell)} \E\braces*{\frac{\indicator\braces*{N(\cell)\neq0}}{N(\cell)}}.
    \end{align*}
    Rearranging this gives
    \begin{align*}
        \E\braces*{\frac{\indicator\braces*{N(\cell)\neq0}}{N(\cell)}} & \geq \frac{\P\braces*{N(\cell)\neq0}}{\E N(\cell)} \\
        & = \frac{1- \P\braces*{N(\cell)=0}}{\E N(\cell)}.
    \end{align*}
    
    The right hand side of \eqref{eq:decomposing_MSE} is bounded below by
    $$
    \Var\braces{f(\bx) ~|~\bx \in \cell} + \E\braces*{f(\bx)~\vline~ \bx \in \cell}^2\P\braces*{N(\cell)=0} + \frac{1- \P\braces*{N(\cell)=0}}{\E N(\cell)}\left(\Var\braces{f(\bx) ~|~\bx \in \cell} + \sigma^2\right).
    $$
    Taking expectation with respect to $\bx$ then gives
    \begin{align*}
        \E_{\data,\bx}\braces*{\left(f(\bx) - \hat f(\bx) \right)^2} & \geq \sum_{\cell \in \partition}\Var\braces{f(\bx) ~|~\bx \in \cell}\nu\braces{\cell} + \frac{1}{n}\sum_{\cell \in \partition}  \left(\Var\braces{f(\bx) ~|~\bx \in \cell} + \sigma^2\right). \\
        & \quad + \sum_{\cell \in \partition} \left(\E\braces*{f(\bx)~\vline~ \bx \in \cell}^2\nu\braces{\cell} - \frac{\Var\braces{f(\bx) ~|~\bx \in \cell} + \sigma^2}{n}\right) (1-\nu\braces{\cell})^n.
    \end{align*}
    Rearranging this gives \eqref{eq:bias_variance_lower_bound}.
\end{proof}


\begin{proof}[Proof of Theorem \ref{thm:bias_variance_permissible}]
    We use the fact that for any cell $\cell$,
    $$
    \paren*{1-\nu\braces{\cell}}^n \leq \paren*{1-\frac{1}{n}}^n \leq \frac{1}{2}.
    $$
    As such, the term $E_2$ in \eqref{eq:bias_variance_lower_bound} is at most
    $$
    E_2 \leq \frac{1}{2n}\sum_{\cell \in \partition} \Var\braces*{f(\bx)~|~\bx \in \cell} + \frac{\abs*{\partition}\sigma^2}{2n}.
    $$
    After performing cancellations, we get \eqref{eq:simplified_bias_variance_lower_bound}. The upper bound follows similarly.
\end{proof}

\section{Proofs for rate-distortion argument} \label{sec:rate_distortion_proofs}

\begin{lemma}[Rates over product distributions] \label{lem:rate_for_product}
    Suppose $\measure = \nu_1 \times \nu_2 \times \cdots \times \nu_d$ is a product distribution on $\mathcal{X} \subset \R^d$. For all $D > 0$, we have
    $$
    R(D;\nu,\bbeta) \geq \inf_{\sum_j \beta_j^2 D_j \leq D } \sum_{j=1}^\dim R(D_j;\nu_j,1).
    $$
\end{lemma}

\begin{proof}
    Let $\hat\bx$ follow the conditional distribution achieving the infimum in the definition of $R(D;\nu,\bbeta)$. Following the calculations in Chapter 10 of \cite{cover2012elements}, we have
    \begin{align*}
        R(D;\nu,\bbeta) & = I(\bx;\hat\bx) \\
        & = \sum_{j=1}^d h(x_j) - \sum_{j=1}^d h(x_j|x_{1:j-1},\hat \bx) \\
        & \geq \sum_{j=1}^d h(x_j) - \sum_{j=1}^d h(x_j|\hat x_j) \\
        & = \sum_{j=1}^d I(x_j;\hat x_j).
    \end{align*}
    Note that $x_j \sim \nu_j$ for $j=1,\ldots,d$. Denoting $\delta_j \coloneqq \E\braces*{\paren{x_j-\hat x_j}^2}$ for each $j$, we can therefore write
    $$
    I(x_j;\hat x_j) \geq R(D_j;\nu_j,1).
    $$
    Finally, observe that
    $$
    \sum_{j=1}^d \beta_j^2 \delta_j = \E\braces*{\norm*{\bx - \hat \bx}_{\bbeta}^2} \leq D,
    $$
    so taking the infimum over possible values of $\delta_j$'s satisfying this constraint gives us the statement of the lemma.
\end{proof}

\begin{lemma}[Rates for dominated weighted norms] \label{lem:rate_for_dominated_norm}
    Let $\bbeta$ and $\bbeta'$ be two vectors such as $\beta_j^2 \geq (\beta_j')^2$ for $j=1,\ldots,\dim$. Then for all $D > 0$, we have
    $$
    R(D;\nu,\bbeta) \geq R(D;\nu,\bbeta').
    $$
\end{lemma}

\begin{proof}
    Obvious from the definition of the rate as an infimum.
\end{proof}

\begin{lemma}[Univariate rates] \label{lem:rate_for_univariate}
    Let $\nu_0$ be a continuous distribution on $\R$. Then we have
    $$
    R(D;\nu_0,1) \geq \paren*{h(\nu_0) - \frac{1}{2}\log\paren*{2\pi e D}} \vee 0.
    $$
    If $\nu_0$ is Bernoulli with parameter $0 < \pi_0 \leq 1/2$, then we have the tighter bound:
    $$
    R(D;\nu_0,1) \geq \paren*{H(\pi_0) - H(D)} \vee 0.
    $$
\end{lemma}

\begin{proof}
    We once again follow the calculations in Chapter 10 of \cite{cover2012elements}. Let $\hat x$ follow the conditional distribution achieving the infimum in the definition of $R(D;\nu_0,1)$. Then \begin{align*}
        R(D;\nu_0,1) & = I(x;\hat x) \\
        & = h(x) - h(x|\hat{x}) \\
        & = h(x) - h(x-\hat{x}|\hat{x}) \\
        & \geq h(x) - h(x-\hat{x}).
    \end{align*}
    Next, we use the maximum entropy property of the normal distribution to write
    $$
    h(x-\hat{x}) \leq \frac{1}{2}\log\paren*{2\pi e D}.
    $$
    Combining this with the observation that mutual information is non-negative completes the proof of the first statement. For the second statement, we repeat the same arguments with discrete entropy, and observe that
    $$
    H(x-\hat{x}) = H(D).
    $$
\end{proof}

\begin{lemma}[Rate bound for Boolean covariates]\label{lem:rate_for_boolean}
    Assume the conditions of Theorem \ref{thm:lower_bound_boolean_additive_model_general}. The rate distortion function may be lowered bounded as
    \begin{equation}\label{eq:optimal_boolean_rate}
        R(D;\nu,\bbeta) \geq \sum_{j \colon \beta^{2}_j \geq \colon m_{\bbeta,\bpi}^{-1}(D)\log((1-\pi_j)/(\pi_j))} H(\pi_j) - H\left(\frac{1}{1+e^{\beta^{2}_{j}/m_{\bbeta,\bpi}^{-1}(D)}}\right)
    \end{equation}
\end{lemma}

\begin{proof}
    Combining Lemmas \ref{lem:rate_for_product} and \ref{lem:rate_for_univariate}, we get 
    \begin{equation} \label{eq:rate_lower_bound_for_boolean_linear_opt}
        R(D;\nu,\bbeta) \geq \inf_{\sum_j \beta_j^2D_j \leq D} \sum_{j=1}^\dim \paren*{ H(\pi_j) - H(D_j)}\vee 0.
    \end{equation}
    The right hand side is equivalent to the solution of the following convex optimization program: 
    $$
    \min~ \sum_{j=1}^\dim  H(\pi_j) - H(\delta_j) \quad \text{s.t.}\quad \sum_{j=1}^\dim \beta_j^2\delta_j \leq D,\quad \delta_j \leq \pi_{j}~\text{for}~j=1,2,\ldots,d.
    $$
    The Lagrangian of this program is
    $$
    L(\bdelta,\blambda) = \sum_{j=1}^\dim H(\pi_{j}) - H(\delta_{j})+ \lambda_0 \paren*{\sum_{j=1}^\dim \beta_j^2\delta_{j} - D} + \sum_{j=1}^\dim \lambda_j(\delta_j - \pi_{j}).
    $$
    Differentiating with respect to $\delta_j$, we get
    $$
    \frac{dL}{d\delta_j} = \log\paren*{\frac{\delta_{j}}{1-\delta_j}} + \lambda_0\beta^{2}_{j} + \lambda_j.
    $$
    Let $\delta_j^*$, $j=1,\ldots,\dim$ and $\lambda_j^*$, $j = 0,\ldots \dim$ denote the solution to KKT conditions. The above equation yields
    $$
    \delta_j^* = \frac{1}{1+e^{\lambda^*_j+\lambda^*_0\beta_j^{2}}}.
    $$
    By complementary slackness,  we have either $\lambda_j^*=0$ or $\delta_j^* = \pi_j$ for each $j$. It is easy to see that this implies
    $$
    \delta_j^* =  \pi_j \wedge \frac{1}{1+e^{\beta_j^{2}/\alpha}} 
    $$
    where $\alpha$ is chosen so that
    $$
    D = \sum_{j=1}^\dim \beta_j^{2}\delta_{j}^* = m_{\bbeta,\bpi}(\alpha).
    $$
    Plugging these values of $\delta_{j}$ into \eqref{eq:rate_lower_bound_for_boolean_linear_opt} gives us \eqref{eq:optimal_boolean_rate}.
\end{proof}

\begin{theorem}[Lower bounds for additive Boolean models]
\label{thm:lower_bound_boolean_additive_model_general}
    Assume the regression model \eqref{eq:regression_model} and that the conditional mean function is linear: $f(\bx) = \bbeta^{T}\bx$. Assume that the covariate space is the hypercube $\hypercube$, and that the covariates are independent, with $x_j \sim \text{Ber}(\pi_j)$, $0 \leq \pi_j \leq \frac{1}{2}$, for $j=1,\ldots,d$. Define the function $m_{\bbeta,\bpi}\colon \paren*{0, \displaystyle\max_j \frac{\beta_j^2}{\log((1-\pi_j)/\pi_j)}} \to \R$ via the formula
    \begin{equation} \label{eq:definition_of_boolean_h}
        m_{\bbeta,\bpi}(\alpha) =  \sum_{j=1}^d \beta_j^2 \paren*{\pi_j \wedge \frac{1}{1+e^{\beta_j^2/\alpha}}}
    \end{equation}
    and notice that it is strictly increasing and hence invertible on its domain. Then we have
    \begin{equation} \label{eq:risk_boolean_lower_bound}
        \oraclerisk(f,\nu,n) \geq \frac{1}{2}\inf_{D > 0} \braces*{D + \frac{\sigma^2 2^{R(D)}}{n}},
    \end{equation}
    where
    $$
    R(D) = \sum_{j \colon \beta^{2}_j \geq \colon m_{\bbeta,\bpi}^{-1}(D)\log((1-\pi_j)/(\pi_j))} H(\pi_j) - H\left(\frac{1}{1+e^{\beta^{2}_{j}/m_{\bbeta,\bpi}^{-1}(D)}}\right).
    $$
    In particular, if $\pi_{j} = \pi$ for $j = 1 \ldots d$, and  $\displaystyle \min_{j \in S} \abs*{\beta_j} \geq \beta_0 > 0$ for some subset of indices $S$ of size $s$, then we have
    \begin{equation} 
        \oraclerisk(f,\nu,n) \geq \frac{s\beta^2_0}{2}\paren*{1 - \paren*{\frac{2e^sn\beta^2_0}{2^{sH(\pi)}\sigma^2}}^{\frac{1}{s-1}}}.
    \end{equation}
\end{theorem}

\begin{proof}
    The first statement in the theorem follows immediately from plugging in the bound from Lemma \ref{lem:rate_for_boolean} into Lemma \ref{lem:rate_distortion_lower_bound}. For the second statement, we first use Lemma \ref{lem:rate_for_dominated_norm} to see that it suffices to bound $R(D;\nu,\tilde\bbeta)$, where $\tilde\beta_j = \beta_0$ for $j \in S$, and $\tilde\beta_j = 0$ for $j \notin S$. One can check that
    $$
    m_{\tilde\bbeta,\bpi}(\alpha) = s\beta_0^2 \paren*{\pi \wedge \frac{1}{1+e^{\beta_0^2/\alpha}}},
    $$
    and so
    $$
    \frac{1}{1+e^{\beta^{2}_{j}/m_{\tilde\bbeta,\bpi}^{-1}(D)}} = \frac{D}{s \beta^2_0} \wedge \pi.
    $$
    Plugging this formula into \eqref{eq:optimal_boolean_rate}, we get
    \begin{equation} \label{eq:rate_distortion_final_bound_for_linear_boolean}
        R(D;\nu,\tilde\bbeta) \geq s \paren*{H(\pi) - H\paren*{\frac{D}{s\beta^{2}_{0}}}} \vee 0.
    \end{equation}
    
    For $\frac{D}{s\beta_0^2} \leq \pi$, we expand
    \begin{align} \label{eq:loose_rate_bound_boolean}
        2^{sR(D;\measure,\bbeta)} & \geq \paren*{\frac{D}{s\beta_0^2}}^{\frac{D}{\beta_0^2}}\paren*{1- \frac{D}{s\beta_0^2}}^{s\paren*{1-\frac{D}{s\beta_0^2}}} \nonumber \\
        & \geq e^{-s}\paren*{1- \frac{D}{s\beta_0^2}}^s,
    \end{align}
    where the second inequality comes from applying Lemma \ref{lem:small_lemma_for_bool}.
    Optimizing the expression
    $$
    D + \frac{\sigma^2}{n}\paren*{\frac{2^{H(\pi)}}{e}}^s\paren*{1- \frac{D}{s\beta_0^2}}^s
    $$
    in $D$, we see that the minimum is achieved at 
    $$
    D = s\beta^2_0\paren*{1 - \paren*{\frac{2e^sn\beta^2_0}{2^{sH(\pi)}\sigma^2}}^{\frac{1}{s-1}}}.
    $$
    Finally, plugging this into equation  \eqref{eq:rate_distortion_lower_bound_for_linear} completes the proof. 
    
\end{proof}

\begin{lemma} \label{lem:small_lemma_for_bool}
    For any $0 < p \leq \frac{1}{2}$, we have $\paren*{\frac{p}{1-p}}^p \geq e^{-1}$.
\end{lemma}

\begin{proof}
    We compute
    \begin{align*}
        \log\paren*{\frac{p}{1-p}} & = - \log\paren*{\frac{1-p}{p}} \\
        & = - \log\paren*{\frac{1}{p}-1} \\
        & \geq -\frac{1}{p} + 2.
    \end{align*}
    As such, we have
    $$
    \paren*{\frac{p}{1-p}}^p =\exp\paren*{p\log\paren*{\frac{p}{1-p}}} \geq e^{-1}.
    $$
\end{proof}

\section{Proofs for covering argument} \label{sec:covering_proof}

The primary goal of this section is to prove the following more general version of Theorem \ref{thm:lower_bound_additive_models}.

\begin{theorem}[Lower bound for additive models on unit length cube] \label{thm:lower_bound_additive_models_general}
    Assume the regression model \eqref{eq:regression_model}, with $f$ be defined as in \eqref{eq:additive_model}, and assume that the covariate space is the unit length cube $[0,1]^\dim$. Suppose $\phi_j \in C^1([0,1])$ for $j=1,\ldots,\dim$. Let $I_1, I_2,\ldots, I_d \subset [0,1]$ be sub-intervals, and suppose that $\bbeta \in \R^d$ is a vector of non-negative values such that for each $j=1,\ldots,\dim$,
    $$
    \min_{t \in I_j} \abs*{\phi_j'(t)} \geq \beta_j.
    $$
    Denote $\mathcal{K} = \braces*{\bx \colon x_j \in I_j ~\text{for}~j=1,\ldots,d}$. Assume that $\nu$ is a continuous distribution with density $q$, and denote $q_{min} = \displaystyle\min_{\bx \in \mathcal{K}} q(\bx)$. Define the function $g_{\bbeta}\colon [0, \displaystyle\max_j \beta_j] \to \R$ via the formula
    \begin{equation} \label{eq:definition_of_h}
        g_{\bbeta}(\alpha) = \alpha^2\abs*{\braces*{j \colon \beta_j \geq \alpha}} + \sum_{j \colon \beta_j < \alpha} \beta_j^2,
    \end{equation}
    and notice that it is strictly increasing and thus invertible on its domain. Then the oracle expected risk is lower bounded as
    \begin{equation} \label{eq:general_additive_model_lower_bound}
        \oraclerisk(f,\measure,n) \geq \inf_{D>0} \braces*{D + \frac{\mu(\mathcal K)\sigma^2}{4n} \prod_{j \colon \beta_j \geq g_{\bbeta}^{-1}(12D/q_{min}\mu(\mathcal K))}\paren*{\frac{\beta_j}{g_{\bbeta}^{-1}(12D/q_{min}\mu(\mathcal K))}}},
    \end{equation}
    with the convention that $g_{\bbeta}^{-1}(t) = \infty$ whenever $t$ is out of the range of $g_{\bbeta}$. In particular, if $\displaystyle \min_{j \in S} \min_{t \in I_j} \abs*{\phi_{j}'(t)} \geq \beta_0 > 0$ for some subset of indices $S \subset \coordindices$ of size $s$, then we have
    \begin{equation} 
        \oraclerisk(f,\measure,n) \geq s\mu(\mathcal K) \paren*{\frac{\beta_0^2q_{min}}{12}}^{s/(s+2)} \paren*{\frac{\sigma^2}{4n}}^{2/(s+2)}.
    \end{equation}
\end{theorem}

We first work with the uniform measure, and relate the conditional variance over a cell with the weighted sum of its squared side lengths. This is the equivalent of \eqref{eq:lower_bound_conditional_variane_rate} in the rate-distortion argument.

\begin{lemma}[Variance and side lengths] \label{lem:variance_and_sides}
    Let $\uniform$ be the uniform measure on $[0,1]^d$. Let $\cell \subset [0,1]^d$ be a cell. Let $f$ be an additive model as in \eqref{eq:additive_model}, and assume that each component function $\phi_j$ is continuously differentiable with $\beta_j \coloneqq \displaystyle{\min_{a_j \leq t \leq b_j}}\abs*{\phi_j'(t)}$, where $a_j$ and $b_j$ are the lower and upper limits respectively of $\cell$ with respect to coordinate $j$. Then we have
    \begin{equation} \label{eq:variance_and_sides}
        \Var_\mu\braces*{f(\bx)~|~\bx \in \cell} \geq \frac{1}{6}\sum_{j=1}^d \beta_j^2(b_j-a_j)^2.
    \end{equation}
\end{lemma}

\begin{proof}
    Note that $\phi_1(x_1), \ldots, \phi(x_j)$ are independent given the uniform distribution on $\cell$. As such, we have
    $$
    \Var\braces*{f(\bx)~|~\bx \in \cell} = \sum_{j=1}^d \Var\braces*{\phi_j(x_j)~|~\bx \in \cell}.
    $$
    We can then further write
    $$
    \Var\braces*{\phi_j(x_j)~|~\bx \in \cell} = \frac{1}{2}\E\braces*{\paren*{\phi_j(t) - \phi_j(t')}^2}
    $$
    where $t, t'$ are independent random variables drawn uniformly from $[a_j,b_j]$. For fixed $t$, $t'$, we use the mean value theorem together with our lower bound on $|\phi_j'|$ to write
    $$
    \paren*{\phi_j(t) - \phi_j(t')}^2 = \phi_j'(\tilde{t})^2 \paren*{t-t'}^2 \geq \beta_j^2 \paren*{t-t'}^2.
    $$
    Since the expectation of the right hand side satisfies
    $$
    \E\braces*{\beta_j^2 \paren*{t-t'}^2} = \frac{\beta_j^2(b_j-a_j)^2}{6},
    $$
    we immediately obtain \eqref{eq:variance_and_sides}.
\end{proof}

Next, we use this to compute the maximum volume of a cell under a constraint on its conditional variance. This is similar to the argument in the proof of Theorem \ref{thm:linear_model_lower_bound}.

\begin{lemma}[Variance and volume] \label{lem:variance_and_volume}
    Assume the conditions of Lemma \ref{lem:variance_and_sides}, and that 
    $$
    \Var_\mu\braces*{f(\bx)~|~\bx \in \cell} \leq D.
    $$
    Then the volume of $\cell$ satisfies the upper bound
    \begin{equation}
    \mu(\cell) \leq \prod_{j \colon \beta_j \geq g_{\bbeta}^{-1}(6D)}\paren*{\frac{g_{\bbeta}^{-1}(6D)}{\beta_j}}.
    \end{equation}
\end{lemma}

\begin{proof}
    Let $l_j = b_j - a_j$ denote the side length of $\cell$ along coordinate $j$. By Lemma \ref{lem:variance_and_sides}, we have $\sum_{j=1}^d \beta_j^2(b_j-a_j)^2 \leq 6D$. Since $\mu(\cell) = \prod_{j=1}^d l_j$, we therefore need to solve the convex optimization problem:
    $$
    \min~ \prod_{j=1}^d l_j^{-1} \quad \text{s.t.}\quad \sum_{j=1}^d \beta_j^2l_j^2 \leq 6D,\quad l_j \leq 1~\text{for}~j=1,2,\ldots,d.
    $$
    The Lagrangian of this program is
    $$
    L(l,\lambda) = \prod_{j=1}^d l_j^{-1} + \lambda_0 \paren*{\sum_{j=1}^d \beta_j^2l_j^2 - 6D} + \sum_{j=1}^d \lambda_j(l_j - 1).
    $$
    Differentiating with respect to $l_j$, we get
    $$
    \frac{dL}{dl_j} = - l_j^{-2}\prod_{k\neq j} l_k^{-1} + 12\lambda_0 \beta_j^2 l_j + \lambda_j.
    $$
    
    Let $l_j^*$, $j=1,\ldots, d$ and $\lambda_j^*$, $j = 0,\ldots, d$ denote the solution to the KKT conditions. Our goal is to solve for the $l_j^*$'s. The above equation yields
    $$
    (l_j^*)^2 = \frac{1}{12\lambda_0^*\beta_j^2\prod_{j=1}^d l_j^*} - \frac{\lambda_j^* l_j^*}{12\lambda_0^*\beta_j^2}.
    $$
    Notice that the first term is proportional to $\beta_j^{-2}$. Furthermore, by complementary slackness, we have $\lambda_j^* = 0$ or $l_j^* = 1$ for each $j$. If the former holds, then the second term is equal to $0$, so that $l_j^* = \alpha/\beta_j$ for some constant $\alpha$. Putting everything together, we get
    \begin{equation} \label{eq:optimal_side_length}
        l_j^* = \frac{\alpha}{\beta_j} \wedge 1
    \end{equation}
    where $\alpha$ is chosen so that
    $$
    6D = \sum_{j=1}^d \beta_j^2(l_j^*)^2 = g_{\bbeta}(\alpha).
    $$
    This implies that
    $$
    \mu(\cell) \leq \prod_{j=1}^d l_j^* = \prod_{j \colon \beta_j \geq g_{\bbeta}^{-1}(6D)}\paren*{\frac{g_{\bbeta}^{-1}(6D)}{\beta_j}}
    $$
    as we wanted.
\end{proof}

The proof of this lemma allows us to interpret the function $g_{\bbeta}$: The value $g_{\bbeta}(\alpha)$ is a bound for the weighted sum of optimal squared side lengths, given the choice of $\alpha$ in \eqref{eq:optimal_side_length}. Hence, for any $D>0$, $g_{\bbeta}^{-1}(D)$ is the value of $\alpha$ that ensures that this weighted sum, and hence the conditional variance, is bounded by the value $D$.

We are now ready to combine these ingredients to prove Theorem \ref{thm:lower_bound_additive_models_general}. The proof proceeds by first using Lemmas \ref{lem:change_of_measure} and \ref{lem:restriction_to_sub_rectangle} to reduce to the case of uniform measure on a subset. Equipped with Lemma \ref{lem:variance_and_volume}, we can then lower bound the size of the partition using a volumetric argument. More detailed calculations involving the $g_{\bbeta}$ function yield the second statement.

\begin{proof}[Proof of Theorem \ref{thm:lower_bound_additive_models_general}]
    Let $\partition$ be any permissible partition. Define
    $$
    D \coloneqq \sum_{\cell \in \partition} \Var_\nu\braces{f(\bx)~|~\bx \in \cell}\nu\braces{\cell}.
    $$
    The goal is now to find a lower bound on $\abs*{\partition}$ in terms of $D$. To do this, we first transform the above bound via Lemma \ref{lem:change_of_measure} and Lemma \ref{lem:restriction_to_sub_rectangle}  to get
    \begin{align*}
        D \geq q_{min}\sum_{\cell \in \partition} \Var_\mu\braces{f(\bx)~|~\bx \in \cell\cap\mathcal{K}}\mu\braces{\cell\cap \mathcal{K}}.
    \end{align*}
    Dividing both sides by $q_{min}\mu(\mathcal{K})$, we get the expression
    $$
    \tilde D \geq \sum_{\cell \in \partition} \tilde p(\cell) \Var_\mu\braces{f(\bx)~|~\bx \in \cell\cap\mathcal{K}}
    $$
    where $\tilde D \coloneqq \frac{D}{q_{min}\mu(\mathcal{K})}$, and the weights $\tilde p(\cell) \coloneqq \frac{\mu(\cell\cap \mathcal{K})}{\mu(\mathcal{K})}$ satisfy $\sum_{\cell \in \partition} \tilde p(\cell) = 1$. By Markov's inequality, we can therefore find a subcollection $\mathfrak{C} \subset \partition$ such that the following two conditions hold:
    \begin{equation} \label{eq:Markov_pt_1}
        \sum_{\cell \in \mathfrak{C}} \tilde p (\cell) \geq \frac{1}{2}
    \end{equation}
    and
    \begin{equation} \label{eq:cond_variance_bound_in_covering_arg}
        \Var_\mu\braces{f(\bx)~|~\bx \in \cell\cap \mathcal K} \leq 2\tilde D
    \end{equation}
    for $\cell \in \mathfrak{C}$. We now proceed as follows. First, we rewrite \eqref{eq:Markov_pt_1} as
    $$
    \sum_{\cell \in \mathfrak{C}} \mu\braces*{\cell \cap \mathcal K} \geq \frac{\mu(\mathcal K)}{2}.
    $$
    Next, noting that each $\cell \cap \mathcal K$, being the intersection of two cells, is itself a cell, we can make use of Lemma \ref{lem:variance_and_volume} and \eqref{eq:cond_variance_bound_in_covering_arg} to get
    $$
    \mu(\cell \cap \mathcal{K}) \leq \prod_{j \colon \beta_j \geq g_{\bbeta}^{-1}(12D/q_{min}\mu(\mathcal K))}\paren*{\frac{g_{\bbeta}^{-1}(12D/q_{min}\mu(\mathcal K ))}{\beta_j}}.
    $$
    Combining the last two statements and rearranging gives
    \begin{equation} \label{eq:lower_bound_partition_size}
        \abs*{\mathfrak C} \geq \frac{\mu(\mathcal K)}{2} \prod_{j \colon \beta_j \geq g_{\bbeta}^{-1}(12D/q_{min}\mu(\mathcal K))}\paren*{\frac{\beta_j}{g_{\bbeta}^{-1}(12D/q_{min}\mu(\mathcal K))}}.
    \end{equation}
    Since the right hand side of \eqref{eq:lower_bound_partition_size} is also a lower bound for $\abs*{\partition}$, we may plug it into \eqref{eq:simplified_bias_variance_lower_bound} to get the first statement of the theorem.
 

    
    For the second statement, it is easy to check that $g_{\bbeta}(\alpha) \geq s\alpha^2$ for $\alpha < \beta_0$, and so we have
    $$
    g_{\bbeta}^{-1}(12D/q_{min}\mu(\mathcal K)) \geq \paren*{\frac{12D}{sq_{min}\mu(\mathcal K)}}^{1/2}.
    $$
    This implies that
    $$
    \prod_{j \colon \beta_j \geq g_{\bbeta}^{-1}(12D/q_{min}\mu(\mathcal K))}\paren*{\frac{\beta_j}{g_{\bbeta}^{-1}(12D/q_{min}\mu(\mathcal K))}} = \paren*{\frac{s\beta_0^2q_{min}\mu(\mathcal K)}{12D}}^{s/2}.
    $$
    
    We plug this into the right hand side of \eqref{eq:general_additive_model_lower_bound} and differentiate to get
    $$
    1 - \frac{s\mu(\mathcal K)\sigma^2}{4n}\paren*{\frac{s\beta_0^2q_{min}\mu(\mathcal K)}{12}}^{s/2}D^{-s/2-1}.
    $$
    Setting this to be zero and solving for $D$ gives
    $$
    D = s\mu(\mathcal K) \paren*{\frac{\beta_0^2q_{min}}{12}}^{s/(s+2)} \paren*{\frac{\sigma^2}{4n}}^{2/(s+2)},
    $$
    which gives us the bound we want.
\end{proof}

\begin{proof}[Proof of Proposition \ref{prop:sparse_additive_upper_bound}]
    Similar to the proof of Lemma \ref{lem:variance_and_sides}, we can show that
    $$
    \Var_\mu\braces*{f(\bx)~|~\bx\in\cell} \leq \frac{\beta_{max}^2}{6} \sum_{j \in S} l_j(\cell)^2
    $$
    where $l_j(\cell)$ is the length of $\cell$ along coordinate $j$. As such, if we set
    $$
    l_j = \paren*{\frac{D}{6s\norm{q}_\infty\beta_{max}^2}}^{1/2}\wedge 1
    $$
    for $j \in S$ and $l_j = 1$ for $j \notin S$, then we have 
    $$
    \Var_\nu\braces*{f(\bx)~|~\bx\in\cell} \leq D.
    $$
    Tessellating the unit cube with cells of these dimensions gives us a valid partition, whose approximation error is upper bounded by $D$. We count how many cells are in this partition using a volumetric argument. Each cell has Euclidean volume at least $\paren*{\frac{D}{6s\norm{q}_\infty\beta_{max}^2}}^{s/2}$. Meanwhile, it is easy to see that the union of the cells is contained in a rectangular region with side lengths equal to 1 in the coordinates with index $j \notin S$, and equal to 2 in the coordinates with index $j \in S$. This means that there are at most $\paren*{\frac{24s\norm{q}_\infty\beta_{max}^2}{D}}^{s/2}$ cells.
    
    Choose
    $$
    D = 24s\norm{q}_\infty\beta_{max}^2\paren*{\frac{\sigma^2}{n}}^{2/(s+2)}.
    $$
    Then, the second term in \eqref{eq:simplified_bias_variance_upper_bound} is bounded by 
    $$
    \frac{6\sigma^2}{n}\paren*{\frac{24s\norm{q}_\infty\beta_{max}^2}{D}}^{s/2} \leq \frac{6\sigma^2}{n} \paren*{\frac{n}{\sigma^2}}^{\frac{2}{s+2}\cdot\frac{s}{2}} = 6\paren*{\frac{\sigma^2}{n}}^{2/(s+2)}.
    $$
    Finally, to take care of the error term $E(\partition)$, we compute
    \begin{align*}
        \paren*{1-\nu\braces{\cell}}^n & = \paren*{1 - \paren*{\frac{D}{24s\beta_{max}}}^{s/2}}^n \\
        & = \paren*{1 - \paren*{\frac{\sigma^2}{n}}^{s/(s+2)}}^n \\
        & \leq \exp\paren*{- n\paren*{\frac{\sigma^2}{n}}^{s/(s+2)}} \\
        & = \exp\paren*{ - \sigma^{2s/(s+2)}n^{2/(s+2)}}.
    \end{align*}
    This allows us to bound
    $$
    E(\partition) \leq \norm{f}_{\infty}^2 \exp\paren*{ - \sigma^{2s/(s+2)}n^{2/(s+2)}}.
    $$
    Plugging these values into \eqref{eq:simplified_bias_variance_upper_bound} completes the proof.
\end{proof}

\begin{lemma}[Change of measure] \label{lem:change_of_measure}
    Let $\measure$ be a distribution on $[0,1]^\dim$. Let $\cell \subset [0,1]^\dim$ be a cell such that $\measure$ has density $q(x)$ on $\cell$ satisfying
    $\displaystyle\min_{\bx \in \cell} q(\bx) \geq q_{min}$. Then we have
    \begin{equation}
        \Var_\nu\braces*{f(\bx)~|~\bx \in \cell}\nu\braces{\cell} \geq q_{min}\Var_\mu\braces*{f(\bx)~|~\bx \in \cell}\mu\braces{\cell}.
    \end{equation}
\end{lemma}

\begin{proof}
    We compute
    \begin{align*}
        \Var_\nu\braces*{f(\bx)~|~\bx \in \cell}\nu\braces{\cell} & = \int_{\bx \in \cell} \paren*{f(\bx) - \E\braces*{f(\bx)~|~\bx \in \cell}}^2 q(\bx) d\bx \\
        & \geq q_{min} \int_{\bx \in \cell} \paren*{f(\bx) - \E\braces*{f(\bx)~|~\bx \in \cell}}^2 d\bx \\
        & \geq q_{min} \int_{\bx \in \cell} \paren*{f(\bx) - \E\braces*{f(\bx)~|~\bx \in \cell}}^2 d\bx \\
        & = q_{min}\Var_\mu\braces*{f(\bx)~|~\bx \in \cell}\mu\braces{\cell}.
    \end{align*}
\end{proof}

\begin{lemma}[Restricting to sub-rectangle] \label{lem:restriction_to_sub_rectangle}
    Let $\measure$ be a distribution on $[0,1]^d$. Let $\cell_1 \subset \cell_2 \subset [0,1]^d$ be nested cells. Then
    \begin{equation}
        \Var_\nu\braces*{f(\bx)~|~\bx \in \cell_1}\nu\braces{\cell_1} \leq \Var_\nu\braces*{f(\bx)~|~\bx \in \cell_2}\nu\braces{\cell_2}
    \end{equation}
\end{lemma}

\begin{proof}
    This is proved similarly to the previous lemma.
\end{proof}

\end{document}